\documentclass[letterpaper, 10 pt, conference]{ieeeconf}  

\usepackage[dvipsnames]{xcolor}
\usepackage{amsmath}
\usepackage{cite}
\usepackage{algorithmic}
\usepackage{graphicx}
\usepackage{amssymb}
\usepackage{placeins}
\usepackage{float}
\usepackage{multirow}
\usepackage{setspace}
\usepackage[normalem]{ulem}
\usepackage{hyperref}

\usepackage[normalem]{ulem}
\usepackage{mathtools}

\usepackage[switch]{lineno} 
\definecolor{dustyred}{RGB}{139, 55, 55}

\newtheorem{definition}{Definition}
\newtheorem{theorem}{Theorem}
\newtheorem{example}{Example}

\newtheorem{corollary}{\textcolor{black}{Corollary}}

\newcommand{\xnonlin}{\xi}
\newcommand{\xlin}{\tilde\xi } 
\newcommand{\x}{\xi }

\newcommand{\revision}[1]{#1}

\newtheorem{remark}{Remark}
\IEEEoverridecommandlockouts                        
\title{\LARGE \bf
Bio-Inspired Event-Based Visual Servoing for Ground Robots
}

\author{
Maral Mordad$^{1,*}$, 
Kian Behzad$^{1,*}$, 
Debojyoti Biswas$^{2}$, 
Noah J. Cowan$^{2,3,\dagger}$, 
and Milad Siami$^{1,\dagger}$%
\thanks{$^{1}$Department of Electrical \& Computer Engineering,
Northeastern University, Boston, MA 02115 USA 
(Emails: {\tt\small \{mordad.m, behzad.k, m.siami\}@northeastern.edu}).}%
\thanks{$^{2}$Laboratory for Computational Sensing and Robotics, 
Johns Hopkins University, Baltimore, MD 21218 USA 
(Email: {\tt\small dbiswas2@jhu.edu}).}%
\thanks{$^{3}$Department of Mechanical Engineering, 
Johns Hopkins University, Baltimore, MD 21218 USA 
(Email: \tt\small ncowan@jhu.edu).}%
\thanks{$^*$These authors contributed equally to this work.}\thanks{$^{\dagger}$N.~J.~Cowan and M.~Siami share senior authorship.}%
}

\begin{document}

\maketitle
\thispagestyle{empty}
\pagestyle{empty}

\allowdisplaybreaks
\begin{abstract}
Biological sensory systems are inherently adaptive, filtering out constant stimuli and prioritizing relative changes, likely enhancing computational and metabolic efficiency. Inspired by active sensing behaviors across a wide range of animals, this paper
\revision{introduces a principled 1D} event-based visual servoing framework for ground robots \revision{operating in structured environments}. Utilizing a Dynamic Vision Sensor (DVS), we demonstrate that by applying a fixed spatial kernel to the asynchronous event stream generated from structured logarithmic intensity-change patterns, the resulting net event flux analytically isolates specific \revision{combinations of kinematic states}. We establish a generalized theoretical bound for this event rate estimator and show that linear and quadratic spatial profiles isolate the robot's velocity and position-velocity product, respectively. Leveraging these properties, we employ a multi-pattern stimulus to directly synthesize a nonlinear state feedback term entirely without traditional state estimation. To overcome the inescapable loss of linear observability at equilibrium inherent in event sensing, we propose a bio-inspired active sensing limit-cycle controller. Experimental validation on a 1/10-scale autonomous ground vehicle confirms the efficacy, extreme low-latency, and computational efficiency of the proposed direct-sensing approach.
\end{abstract}




\section{Introduction}\label{sec:introduction}
Biological sensors function primarily as relative detectors that attenuate constant (i.e., low-frequency or ``DC'') stimuli while emphasizing changes (``AC'') \cite{taylor2007sensory}, a characteristic known as sensory adaptation.
To prevent stationary objects from becoming  ``invisible'' due to these high-pass sensors, animals execute continuous bursts of back-and-forth swimming motions \cite{stamper2012active,biswas2018closed,biswasmode2023} that convert static spatial gradients into detectable temporal signals \cite{nelson1997characterization}. Examples include electric fish
\cite{stamper2012active,biswas2018closed,chen2020tuning}, rodents actively whisking to induce friction-based micro-accelerations \cite{mitchinson2011active}, insects relying on flight-driven optic flow \cite{taylor2007sensory}, blind cave fish gliding to distort pressure waves \cite{windsor2010flow},
and the human visual system employing involuntary microsaccades to prevent perceptual fading \cite{rucci2007miniature}. \revision{While absolute quantities (e.g., position) can be represented at the perceptual level, early sensory processing primarily encodes changes, so absolute information is not directly available at the sensory level \cite{barlow1961possible}.}

Such adaptive, high-pass sensing, while highly efficient and low-latency, poses significant design challenges for robotics. For example, in mobile robotics, conventional state estimation and control methods fail under these conditions, as proved in \cite{biswas2025exact}. Traditional sensors, such as frame-based cameras, IMUs, and LiDAR, rely on power-hungry, absolute sensing. They capture data at fixed intervals regardless of whether a scene is dynamic or static. For this reason, replicating bio-inspired principles is crucial for future developments, which requires a fundamental shift from these traditional methods to relative, event-driven sensing \cite{pfeifer2006body, pfeifer2006morphological,gallego2020event}.

In this work, we employ a neuromorphic event camera, also known as a Dynamic Vision Sensor (DVS). Serving as a technological analog to biological vision, the DVS samples light based on scene dynamics rather than an arbitrary external clock \cite{gallego2020event}. Instead of capturing full images at fixed frame rates, each pixel operates independently and asynchronously. \revision{A pixel triggers a ``spike'' (an address-event) only when its local logarithmic brightness changes by a preset threshold, producing a continuous, asynchronous stream of event tuples, rather than synchronously generated frames at fixed rates. Each spike contains the pixel coordinates, a microsecond-resolution timestamp, and the polarity $(\pm)$ of the brightness change \cite{kragic2002survey}.} Consequently, the DVS offers several key advantages over standard cameras, including microsecond temporal resolution, high dynamic range (HDR), and immunity to motion blur \cite{gallego2020event, posch2014retinomorphic}. Furthermore, by transmitting data exclusively during motion, the DVS performs inherent hardware-level compression, effectively filtering out redundant static backgrounds.

Motivated by these biological principles, we aim to develop \revision{a principled approach} Event-Based Visual Servoing (EBVS) for embodied intelligent agents. \revision{While future applications include multi-dimensional navigation in complex visual scenes for drones and collaborative robots, this initial work focuses on validating the core theory within a constrained 1D structured environment with a structured visual scene.} EBVS entails moving either an event camera or the camera’s visual target such that the generated event stream drives the camera to a desired region~\cite{muthusamy2021neuromorphic}.

\revision{Image-Based Visual Servoing (IBVS)  traditionally relies on frame-based cameras to extract and match geometric features against a stored ``desired'' view \cite{cowan2002visual, hutchinson2002tutorial, kragic2002survey},\revision{\cite{xing2024field}}. 
IBVS cannot be applied directly to an event stream because any desired, static view would be event-free. Nevertheless, prior work has begun closing this gap.} 
To reduce computational overhead, Garcia et al.~\cite{garcia2013event} emulated an event-based camera to trigger IBVS updates only when tracked features crossed spatial boundaries. Similarly, Gil et al.~\cite{gil2014active} utilized event cameras for high-speed feature tracking, falling back to frame-based vision when necessary. Pushing toward fully event-driven perception, Muthusamy et al.~\cite{muthusamy2021neuromorphic} replaced conventional vision processing with event-based feature tracking while retaining traditional IBVS control laws. More recently, learning-based approaches have emerged; for example, Vinod et al.~\cite{vinod2025sebvs} employed deep neural networks to map fused RGB images and simulated event streams directly to motor commands.

In this work, we propose a purely event-driven and computationally efficient EBVS framework. \revision{Our approach extends Kernel-Based Visual Servoing (KBVS) \cite{kallem2007kernel, swensen2010empirical} to event-based sensing. KBVS bypasses explicit feature extraction and segmentation by applying a spatial kernel to process the image as a continuous signal. Here, we apply this kernel-based formulation to asynchronous event streams generated by predefined intensity patterns.} The objective is to regulate the camera position to a desired region of the displayed pattern. Our approach avoids explicit feature extraction, feature tracking, and state estimation. In fact, it relies solely on the net number of events generated by the event camera, leading to a highly computationally efficient implementation. To maintain system observability and overcome sensory adaptation, we adapt the control-theoretic framework proposed in \cite{biswas2025exact}. This allows us to formulate a direct sensor-to-motor active sensing control law with proven stability conditions.

\revision{Our contribution is a direct, kernel-based EBVS framework that bypasses explicit feature tracking. We establish novel analytical bounds for net event counts over structured spatial patterns and synthesize nonlinear feedback entirely within the sensor domain. Our controller extends the active-sensing limit-cycle control introduced in \cite{biswas2025exact} to the physical dynamics of an autonomous ground vehicle.}\footnote{See GitHub repository for video, code, and datasets associated with this work: \url{https://github.com/SiamiLab/BioInspiredEBVS}.}.

\section{Problem Formulation and Modeling} \label{sec:problem_formulation_and_modeling}

In this section, we describe the sensing and control problem for the proposed EBVS system and outline the experimental setup, assumptions, and modeling used in this study. The experimental platform consists of a Prophesee Evaluation Kit 4 (EVK4) \cite{prophesee_evk4}, mounted on a ground vehicle. The vehicle used in this setup is the Quanser QCar \cite{QCar_Quanser}, a 1/10-scale autonomous research platform.

In our setup, the robot is constrained to move along the $x$-axis of its body frame, while the camera faces a monitor displaying a horizontally exponential intensity profile. The overall experimental setup is illustrated in Figure~\ref{fig:overall_setup}. The camera frame is defined such that the $z$-axis points forward, the $x$-axis points to the left, and the $y$-axis points upward. Although the camera and robot frames are located at different points, their $x$-axes are aligned. This alignment is adopted to simplify the notation. Consequently, the variable $x(t)$ denotes both the camera position and the robot position throughout the paper. The objective is to stabilize the camera position $x(t)$ at the desired region of the displayed pattern.

\begin{figure*}[t]
    \centering
    \includegraphics[width=0.985\linewidth]{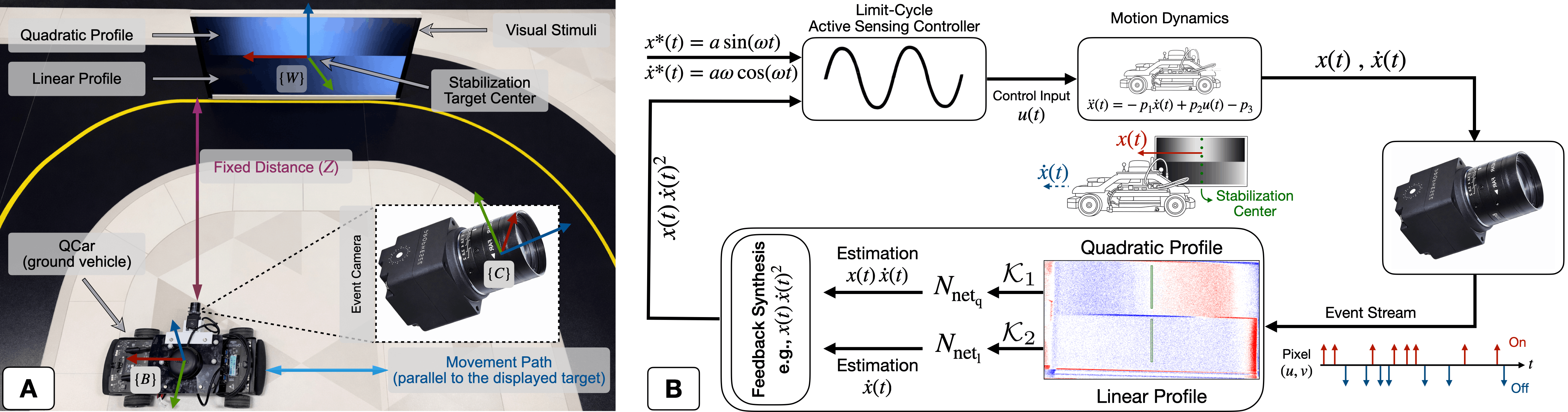}
    \caption{\revision{\textbf{Experimental setup and control architecture.} \textbf{(A)} A ground vehicle moves parallel to a monitor displaying quadratic and linear intensity patterns for an event camera. RGB axes denote the $x, y, z$ coordinate frames for the robot, camera, and world (origin at the pattern's stabilization target). \textbf{(B)} Closed-loop EBVS with active sensing: Net event counts from the stimuli regulate vehicle-driven camera motion to stabilize it at the desired center.}}
    \label{fig:overall_setup}
\end{figure*}

\subsection{Event Camera Model and Event Processing}
\label{sec:event_camera_model_and_event_processing}

We model the event camera using a standard pinhole projection. Let $^{C}\mathbf{p} = [p_x, p_y, p_z]^\top$ denote a 3D point expressed in the camera coordinate frame. The projection of this point onto the image plane with coordinates $^{I}\mathbf{u} = [u, v]^\top$ is given by
\begin{equation*}
    u = -f_x \frac{p_x}{p_z} + o_x, \quad v = -f_y \frac{p_y}{p_z} + o_y,
\end{equation*}
where $f_x$ and $f_y$ are the focal lengths in pixels, and $(o_x, o_y)$ denotes the principal point. Here, $(u,v)$ represent continuous image-plane coordinates, while the camera sensor samples this continuous domain at discrete pixel locations. The image coordinate system is defined such that $u$ increases to the right and $v$ increases downward.

In our setup, the camera moves strictly along the camera's $x$-axis, while maintaining a constant distance $Z$ from the monitor. As the camera moves, the visual pattern shifts laterally across the image plane. The light intensity $I(u,v,t)$ observed by a pixel at $(u,v)$ at time $t$ can be expressed as
\begin{equation}
    I(u,v,t) = I_0\left(u + \mu(t), v\right),
    \label{eq:pinhole_model_x_motion}
\end{equation}
where $I_0(u, v)$ indicates the light intensity of pixel $(u,v)$ when the camera is located exactly at the origin ($x = 0$). Based on the pinhole camera model $\mu(t) = -\frac{f_x}{Z}x(t)$, which represents the horizontal shift of the image pattern on the sensor induced by the camera translation $x(t)$.

In this work, we assume that the robot can move alongside a scene that exhibits an arbitrary horizontally exponential intensity profile. Accordingly, the reference intensity pattern $I_0(u,v)$ is modeled as
\begin{equation}
I_0(u, v) = \exp\!\big(f(u)\big), \quad \forall v,
\label{eq:I_0_profile}
\end{equation}
where $f(u)$ represents the horizontal logarithmic intensity profile.
Unlike a conventional camera, which directly measures absolute intensity values as in \eqref{eq:pinhole_model_x_motion} and \eqref{eq:I_0_profile}, an event camera generates asynchronous events at the pixel level in response to changes in light intensity. An event $e = (u, v, p, t)$ is triggered when the change in the logarithm of the light intensity at a pixel exceeds a predefined threshold. Here, $p \in \{-1, 1\}$ indicates the polarity of the brightness change (decrease or increase), and $t$ is the timestamp recorded with microsecond resolution. Mathematically, this condition is expressed as
\begin{equation}
    |\Delta \log I(u,v)| = |\log I(u,v, t + \Delta t) - \log I(u,v, t)| \geq C,
\end{equation}
where we assume the threshold for both positive and negative events has the same magnitude $C$.

\begin{definition}[Net Event Count] \label{dfn:net_summation}
Let $\mathcal{K} \subset \mathbb{Z}^2$ denote a specified region of pixels on a sensor array. Over a time interval $[t, t+\Delta t]$, let $N_{\mathrm{pos}}$ and $N_{\mathrm{neg}}$ denote the total number of positive ($p = 1$) and negative ($p = -1$) events triggered within $\mathcal{K}$, respectively. The \textbf{net event count} $N_{\mathrm{net}}$ is defined as $    N_{\mathrm{net}} \triangleq N_{\mathrm{pos}} - N_{\mathrm{neg}}$.

Under the ideal generative event model, this discrete count is related to the logarithmic intensity change by
\begin{equation*}
    N_{\mathrm{net}} = \sum_{(u, v) \in \mathcal{K}} \frac{\Delta \log I(u,v)}{C},
\end{equation*}
where $I(u,v)$ is the pixel intensity, $\Delta \log I(u,v)$ is the change in logarithmic intensity over the interval $\Delta t$, and $C > 0$ is the contrast threshold.
\end{definition}

\begin{theorem}
\label{thm:net_count_bound}
Let $x(t) \in C^2(\mathbb{R})$ denote the camera position with instantaneous linear velocity $\dot{x}(t)$ and acceleration $\ddot{x}(t)$. Assume the camera observes a scene with an arbitrary horizontally exponential intensity profile $I_0(u, v) = \exp(f(u)), \: \forall v$, where $f \in C^3(\mathbb{R})$. Let the event-counting kernel $\mathcal{K}$ consist of $N_u \times N_v$ pixels and be horizontally symmetric with respect to the principal point. During a sufficiently small time interval $\Delta t$, the net event count $N_{\mathrm{net}}$ is bounded by%
\begin{equation}
|N_{\mathrm{net}} - M \Delta t| \leq L_{\mathrm{time}} \Delta t^2 + L_{\mathrm{space}} \Delta t,
\label{eq:net_event_estimator}
\end{equation} where the net event rate estimator $M$ is evaluated at the image-plane coordinate $\mu(t)=-\frac{f_x}{Z}x(t)$ as
\begin{equation*}
M \triangleq -\frac{N_u N_v f_x}{C Z} \dot{x}(t) f'(\mu(t)).
\end{equation*}

The estimation error is bounded by a temporal constant $L_{\mathrm{time}}$ and a spatial constant $L_{\mathrm{space}}$, defined as
\begin{align*}
L_{\mathrm{time}}
&\triangleq \left( \tfrac{N_u N_v}{2 C} \right)
\Bigg[
F^{(2)}_{\mathrm{max}} \left( \tfrac{f_x}{Z} v_{\mathrm{max}} \right)^2 + F^{(1)}_{\mathrm{max}} \left( \tfrac{f_x}{Z} a_{\mathrm{max}} \right)
\Bigg], \\
L_{\mathrm{space}}
&\triangleq \left( \tfrac{N_u N_v}{C} \right)
\left( \tfrac{N_u^2}{8} \right)
\left( \tfrac{f_x}{Z} v_{\mathrm{max}} \right)
F^{(3)}_{\mathrm{max}},
\end{align*} 
where $v_{\mathrm{max}}$ and $a_{\mathrm{max}}$ denote the maximum absolute velocity and acceleration of the robot, respectively. $F^{(n)}_{\mathrm{max}} = \sup |f^{(n)}(u)|$ represents the suprema of the spatial derivatives of the logarithmic intensity over the observed interval.
\end{theorem}

\begin{proof}    
For a small interval $\Delta t$, the change in logarithmic intensity at pixel coordinate $(u,v)$ can be expanded using a first-order temporal Taylor expansion
$\Delta \log I(u,v) = \frac{d}{dt} \log I(u,v,t)\, \Delta t + R_2(t, \Delta t)$,
where $R_2(t, \Delta t)$ denotes the temporal Lagrange remainder \cite{dym2004principles}. Using \eqref{eq:pinhole_model_x_motion} and \eqref{eq:I_0_profile}, and applying the chain rule, the temporal derivative becomes
$\frac{d}{dt} \log I(u,v,t) = f'(u + \mu(t))\, \dot{\mu}(t)$.
Because $f'$ is evaluated at an offset $u$ from $\mu(t)$, we perform a second-order spatial Taylor expansion of $f'$ around $\mu(t)$ as
$f'(u + \mu(t)) = f'(\mu(t)) + u f''(\mu(t)) + S_2(u, t)$, where $S_2(u, t)$ is the spatial Lagrange remainder. Substituting this into the temporal derivative and summing over all $N_u \!\times\! N_v$ pixels in the kernel $\mathcal{K}$ yields the net event count as
\[
\scalebox{0.8}{
$
\begin{aligned}
N_{\mathrm{net}} &= \frac{1}{C} \sum_{(u,v) \in \mathcal{K}} \Big[ \dot{\mu}(t) \big( f'(\mu(t)) + u f''(\mu(t)) + S_2(u, t) \big) \Delta t + R_2 \Big] \\
&= M \Delta t + \tfrac{\dot{\mu}(t) \Delta t f''(\mu(t))}{C}\! \sum_{(u,v) \in \mathcal{K}}\! u + \tfrac{\dot{\mu}(t) \Delta t}{C}\! \sum_{(u,v) \in \mathcal{K}}\! S_2 + \tfrac{1}{C}\! \sum_{(u,v) \in \mathcal{K}}\! R_2.
\end{aligned}
$
}
\]


Because the kernel $\mathcal{K}$ is horizontally symmetric, the linear sum of horizontal coordinates vanishes ($\sum_{(u,v) \in \mathcal{K}} u = 0$). Isolating the estimation error leaves the two remainder sums as
$|N_{\mathrm{net}} - M \Delta t| \le \frac{|\dot{\mu}(t)| \Delta t}{C} \sum_{(u,v) \in \mathcal{K}} |S_2| + \frac{1}{C} \sum_{(u,v) \in \mathcal{K}} |R_2|$.

To bound the spatial error, based on the Lagrangian remainder $S_2(u, t) = \frac{u^2}{2} f'''(\xi^*)$ for some $\xi^* \in (\mu(t), \mu(t)+u)$. Note the maximum absolute pixel coordinate is $u_{\mathrm{max}} = \frac{N_u}{2}$, ensuring $\frac{u^2}{2} \le \frac{N_u^2}{8}$. Applying the supremum $F^{(3)}_{\mathrm{max}}$ across $N_u N_v$ pixels yields $L_{\mathrm{space}} \Delta t$. Similarly to bound the temporal error, $R_2(t, \Delta t) = \frac{\Delta t^2}{2} [f''(u + \mu(\epsilon^*))(\dot{\mu}(\epsilon^*))^2
+ f'(u + \mu(\epsilon^*))\,\ddot{\mu}(\epsilon^*)]$ for some $\epsilon^* \in (t, t+\Delta t)$. Maximizing this over the interval with $v_{\mathrm{max}}$, $a_{\mathrm{max}}$, $F^{(1)}_{\mathrm{max}}$, and $F^{(2)}_{\mathrm{max}}$ yields $L_{\mathrm{time}} \Delta t^2$. This establishes the generalized bounds and completes the proof. 
\end{proof}

\begin{example}[Quadratic Profile] 
\label{example:gaussian_profile}
For an intensity with a quadratic profile $I_0(u, v) = \exp(-\frac{u^2}{2\sigma^2})$ for some non-zero constant $\sigma \in \mathbb{R}$, the net event estimates the product of camera position and velocity. Defining $f(u) = - \frac{u^2}{2\sigma^2}$, the spatial derivatives are $f'(u) = -\frac{u}{\sigma^2}$ and $f''(u) = -\frac{1}{\sigma^2}$. Notably, $f'''(u) = 0$. Substituting $f'(\mu(t)) = \frac{f_x}{Z \sigma^2} x(t)$ into $M$ yields
\begin{equation}
M_{\mathrm{Quadratic}} 
= \underbrace{-\tfrac{N_u N_v f_x^2}{C Z^2 \sigma^2}}_{:= C_q}\, x(t)\dot{x}(t).
\label{eq:M_Quadratic}
\end{equation}

Since $F^{(3)}_{\mathrm{max}} = 0$, the spatial bound is $L_{\mathrm{space}} = 0$, safely recovering the purely temporal bound for the quadratic profile. Using $f'(u)$ and $f''(u)$, the temporal bound $L_{\mathrm{time}}\Delta t^2$ can be explicitly evaluated, and by normalizing with $C_q \Delta t$, this yields the bound $\revision{\varepsilon_q = }\left(\frac{v_{\max}^2}{2} + \frac{Z N_u a_{\max}}{2 f_x}\right)\Delta t$ on the estimate of $x(t)\dot{x}(t)$ itself.

\end{example}

\begin{example}[Linear Profile]
\label{example:linear_profile}
If the scene exhibits a linear profile $I_0(u, v) = \exp(ku)$ for some non-zero constant $k \in \mathbb{R}$, the net event count becomes strictly proportional to the camera's instantaneous velocity. Furthermore, the spatial estimation error vanishes entirely, and the temporal error depends exclusively on the camera's acceleration.
Let $f(u) = ku$. The spatial derivatives are constant or zero: $f'(u) = k$, $f''(u) = 0$, and $f'''(u) = 0$. Substituting $f'(\mu(t)) = k$ into $M$ yields
\begin{equation}
M_{\mathrm{Linear}} = \underbrace{-\tfrac{N_u N_v f_x k}{C Z}}_{:=C_l}\, \dot{x}(t).
\label{eq:M_linear}
\end{equation}

Because the third derivative is exactly zero globally ($F^{(3)}_{\mathrm{max}} = 0$), the spatial error bound evaluates to zero $L_{\mathrm{space}} = 0$. For the temporal error bound, the suprema of the first and second derivatives are $F^{(1)}_{\mathrm{max}} = |k|$ and $F^{(2)}_{\mathrm{max}} = 0$. Substituting these into the  $L_{\mathrm{time}}$ equation eliminates the velocity-squared term, leaving only the acceleration-dependent term as
  $L_{\mathrm{time}} = |k|  \left( \frac{N_u N_v}{2 C} \right) \left( \frac{f_x}{Z} a_{\mathrm{max}} \right)$.

This demonstrates that for a linear profile, a perfectly symmetric kernel introduces no spatial approximation error, and temporal error arises only if the camera is accelerating. By normalizing $L_{\mathrm{time}}\Delta t^2$ with $C_l \Delta t$, we obtain the bound $\revision{\varepsilon_\ell = }\frac{a_{\max}}{2}\Delta t$ on the estimate of $\dot{x}(t)$ itself.
\end{example}

\subsection{Longitudinal Dynamics of the Robot}
\label{subsec:sysid}


Following \cite{caponio2024modeling}, the longitudinal dynamics of the robot (aka the camera's motion) are governed by the DC motor model as
$ \dot{\omega}(t) = \frac{K_t}{J_w R_a} V_a(t) - \left( \frac{K_t K_v}{R_a J_w} + \frac{B}{J_w} \right) \omega(t) - \frac{T_c}{J_w}$,
where $\omega(t)$ is the motor's angular velocity, $V_a(t)$ is the armature voltage, $K_t$ is the torque constant, $J_w$ is the equivalent inertia, $R_a$ is the armature resistance, $K_v$ is the back-emf constant, $B$ is the viscous friction, and $T_c$ represents Coulomb friction.

The armature voltage and motor angular velocity relate to the input PWM duty cycle $u(t)$ and robot linear velocity $\dot{x}(t)$ via $V_a(t) = V_{\mathrm{bat}}u(t)$ and $\dot{x}(t) = \frac{R_w}{\mathcal{G}}\omega(t)$, respectively. Here, $V_{\mathrm{bat}}$ is the nominal battery voltage, $R_w$ is the wheel radius, and $\mathcal{G}$ is the transmission ratio. Substituting these relationships yields the system's longitudinal dynamics
\begin{equation}
    \ddot{x}(t) = -p_1\dot{x}(t) + p_2u(t) - p_3,
\label{eq:qcar_longitudinal_dynamics}
\end{equation}
where $p_1 = \frac{K_t K_v}{R_a J_w} + \frac{B}{J_w}$, $p_2 = \frac{K_t R_w V_{\mathrm{bat}}}{J_w R_a \mathcal{G}}$, and $p_3 = \frac{R_w T_c}{J_w \mathcal{G}}$ are lumped physical constants.

\subsection{Synthesis of Event-Driven Feedback} 
With the robot's longitudinal dynamics established in \eqref{eq:qcar_longitudinal_dynamics}, the core control challenge is to regulate this system utilizing only the asynchronous event streams modeled in Section~\ref{sec:event_camera_model_and_event_processing}. Conventional state observers are ill-suited for this task; formulating continuous-time estimators for asynchronous, relative neuromorphic data is highly complex, and more critically, the plant fundamentally loses local observability as the camera velocity approaches zero. Rather than attempting explicit state reconstruction, we leverage a dual-pattern visual stimulus. By concurrently computing the event flux over a quadratic kernel ($\mathcal{K}_1$) and a linear kernel ($\mathcal{K}_2$), we independently estimate the position-velocity product $x(t)\dot{x}(t)$ via \eqref{eq:M_Quadratic} and the instantaneous velocity $\dot{x}(t)$ via \eqref{eq:M_linear}.  Multiplying these parallel streams allows us to directly synthesize the higher-order nonlinear feedback term, $x(t)\dot{x}(t)^2$, entirely within the sensor domain.  The subsequent section details how this synthesized signal is embedded into a bio-inspired active sensing control law to actively overcome these fundamental observability limitations.

\section{Active Sensing Strategy}
Prior research by Biswas et al.~\cite{biswas2025exact} demonstrated that systems of the form \eqref{eq:qcar_longitudinal_dynamics} with event-based outputs lose observability under linearization. Moreover, that study established that no dynamic output-feedback controller can asymptotically stabilize the origin of this class of systems ~\cite[Proposition~3.1]{biswas2025exact}. To overcome this limitation, the prior work by Biswas et al.~\cite{biswas2025exact} proposed stabilizing the system around a periodic orbit rather than an equilibrium. That paper considered a simplified, dimensionless mass-damper system to a limit cycle of fixed period. Here, we show how to apply that controller to the physical system \eqref{eq:qcar_longitudinal_dynamics} by stabilizing it to the tunable limit cycle as
\begin{equation*}
\label{eq:limit_cycle}
\x^\star(t) = (x^\star(t),\dot{x}^\star(t)) =
\big(a\sin(\omega t),\, a\omega\cos(\omega t)\big),
\end{equation*}
where $a$ is the amplitude and $\omega$ is the oscillation frequency of the periodic motion.
Assuming full knowledge of the parameters $p_i$, $i=1,2,3$ in \eqref{eq:qcar_longitudinal_dynamics}, the control input is chosen as
\begin{equation}
\label{eq:control_law}
\begin{aligned}
u = \tfrac{1}{p_2}(&
p_1 a\omega \cos(\omega t)
- a\omega^2 \sin(\omega t)
+ p_3 \\
&- K (
x(t)\dot{x}(t)^2
- a^3\omega^2 \sin(\omega t)\cos^2(\omega t)
)
).
\end{aligned}
\end{equation}
where $K$ is the output feedback gain.
The first two terms provide feedforward excitation that counteracts damping and sustains informative motion necessary for observability~\cite{biswas2025exact},  while the constant term \(p_3\) compensates the friction term in \eqref{eq:qcar_longitudinal_dynamics}. The final nonlinear feedback terms vanish on the target orbit and penalize deviations from it. 
\revision{Together, these components render the periodic orbit $\xi^\star(t)$ locally orbitally stable.}

Linearization of the system \eqref{eq:qcar_longitudinal_dynamics} with input \eqref{eq:control_law} around $\x^*(t)$ yields
\begin{gather}
\label{eq:LTV}
  \dot \xlin = A(t)\xlin,
  \text{ where }~\xlin:=\xnonlin-\x^*\\ 
  \label{eq:A_def}
  A(t):=\begin{bmatrix}
    0 &1\\
    -\delta\omega^2\cos^2(\omega t) &-p_1-\delta\omega\sin(2\omega t)\\
  \end{bmatrix},\end{gather} 
with $\xnonlin = (x,\;\dot x)^{\top}$ and the parameter $\delta$ defined as $\delta:= Ka^2$.

The linearized dynamics are time-periodic with period \(\pi/\omega\) in the original time variable \(t\) (equivalently, \(\pi\)-periodic in the rescaled time \(\tau=\omega t\)).


\begin{corollary}\textit{(Exponential orbital stability, extension from Biswas et al.~\cite{biswas2025exact})}
\label{cor:exp_stable}
    The origin of the system \eqref{eq:LTV}  is exponentially stable for $0<\delta\leq  \frac{\sqrt{\omega^2+4p_1^2}-\omega}{2\omega}=:\delta^\dagger$. \revision{Consequently, the periodic orbit $\xi^\star(t) = (x^\star(t), \dot{x}^\star(t))$ of the original nonlinear system is locally exponentially orbitally stable.}
\end{corollary}
\begin{proof}
Consider a quadratic Lyapunov candidate function, $V(\xlin(t))$ given by
\begin{equation*}
    \!\!V(\xlin(t)) = \dfrac{1}{2}\xlin(t)^\top P \xlin(t),\,
    \text{}\, P = \begin{bmatrix}
        1 &\tfrac{1}{p_1} \\
       \tfrac{1}{p_1} &\eta
    \end{bmatrix},\, \eta p_1^2>1.
\end{equation*}

The derivative of $V$ along the trajectories of the linear system \eqref{eq:LTV} is given by
\begin{align*}
    \dot V(\xlin(t)) 
                  = -\xlin(t)^\top Q(t)\xlin(t),
\end{align*}
where $Q(t):=-\frac{1}{2}(PA(t)+A(t)^\top P)$ is
\begin{equation*}
    \begin{bmatrix}
\tfrac{\delta\omega^2}{p_1}\cos^2(\omega t)
&
\tfrac{\delta}{2} q_{12}(t)
\\
\tfrac{\delta}{2} q_{12}(t)
&
\eta p_1-\tfrac{1}{p_1}+\eta\delta\omega\sin(2\omega t)
\end{bmatrix},
\end{equation*}
with $ q_{12}(t) = (\omega/p_1)\sin(2\omega t)+\eta\omega^2\cos^2(\omega t).$
$Q(t)$ is positive semidefinite if both the trace and determinant are nonnegative. Imposing nonnegativity of the trace, we have
\begin{align*}
     \mathrm{Tr}(Q)
    &=
    \tfrac{\delta\omega^2}{p_1}\cos^2(\omega t)
    +
    \eta p_1-\tfrac{1}{p_1}
    +
    \eta\delta\omega\sin(2\omega t)
    \nonumber\\
    &\ge
    \eta p_1-\tfrac{1}{p_1}-\eta\delta\omega.
\end{align*}
Thus if $\delta \le \frac{\eta p_1^2-1}{\eta p_1\omega}$ then $\mathrm{Tr}(Q)\geq 0\,\forall t$. 
The determinant is
\begin{equation*}
\mathrm{Det}(Q)
= -\tfrac{\delta \omega^2}{4 p_1^2} \cos^2(\omega t)\, B(t),
\end{equation*}
where
\[
B(t):=
4(1-\eta p_1^2)
+\delta \left( \eta \omega p_1 \cos(\omega t) - 2 \sin(\omega t) \right)^2.
\]
Since $-\tfrac{\delta \omega^2}{4 p_1^2} \cos^2(\omega t)\le 0\, \forall t$, with equality only at $t=k\pi/\omega$, $k\in\mathbb{Z}$, it suffices to ensure that $B(t)\le 0\, \forall t$. Using
\[
\left( \eta \omega p_1 \cos(\omega t) - 2 \sin(\omega t) \right)^2
\le (\eta p_1\omega)^2+4,
\]
we obtain
\[
B(t)\le -4(\eta p_1^2-1)+\delta\bigl((\eta p_1\omega)^2+4\bigr).
\]
Therefore, a sufficient condition for \(B(t)\le 0\) is
\begin{equation}
    \delta
    \le
    \frac{4(\eta p_1^2-1)}{\eta^2\omega^2 p_1^2+4}.
\label{eq:det_cond}
\end{equation}
Since $\eta p_1\omega>0$ and $(\eta p_1\omega-2)^2\ge 0$, we have
\begin{equation*}
\frac{4}{(\eta p_1\omega)^2+4}\le \frac{1}{\eta p_1\omega}.
\end{equation*}
Therefore,
\[
\frac{4(\eta p_1^2-1)}{(\eta p_1\omega)^2+4}
\le
\frac{\eta p_1^2-1}{\eta p_1\omega},
\]
and hence the determinant condition \eqref{eq:det_cond} implies the trace condition.
Maximizing the right-hand side of \eqref{eq:det_cond} over $\eta$ yields
\begin{equation*}
\eta^\dagger
=
\frac{1}{p_1^2}
\left(
1+\sqrt{1+\frac{4p_1^2}{\omega^2}}
\right),
\end{equation*}
with the corresponding bound
\begin{equation*}
\delta^\dagger
=
\frac{\sqrt{\omega^2+4p_1^2}-\omega}{2\omega}.
\end{equation*}
Thus, choosing $\eta=\eta^\dagger$ in our
candidate Lyapunov function, then if $0<\delta\leq \delta^\dagger$, the conditions on trace and determinant will be satisfied. 

Let $q(t):=\lambda_{\min}(Q(t))\ge 0$. Since $Q(t)$ is continuous and $T$-periodic, with $T=\pi/\omega$, the function $q(t)$ is also continuous and $T$-periodic. Moreover, because $q(t)$ vanishes only at isolated times, it is not identically zero on $[0,T]$, and hence
\begin{equation}
    \mu:=\int_0^T q(s)\,ds>0.
\label{eq:mu_def}
\end{equation}
Now, since
\[
\dot V = -\xlin^\top Q(t)\xlin \le -q(t)|\xlin|^2
\le -\frac{q(t)}{\lambda_{\max}(P)}\,\xlin^\top P\xlin,
\]
the quantity $g(t):=\frac{1}{2}\xlin(t)^\top P \xlin(t)$ satisfies
\[
\dot g(t)\le -\frac{2q(t)}{\lambda_{\max}(P)}\,g(t).
\]
Therefore,
\begin{equation*}
g(t)\le g(0)\exp\!\left(-\frac{2}{\lambda_{\max}(P)}\int_0^t q(s)\,ds\right).
\end{equation*}
Using periodicity and the inequality $\lfloor c\rfloor\geq c-1$,
\[
\int_0^t q(s)\,ds \ge \Big\lfloor \tfrac{t}{T}\Big\rfloor \mu
\ge \frac{\mu}{T}t-\mu,
\]
so
\begin{equation*}
g(t)\le e^{\mu/\lambda_{\max}(P)}g(0)\exp\!\left(-\frac{2\mu}{T\lambda_{\max}(P)}t\right).
\end{equation*}

Since $\lambda_{\min}(P)|\xlin|^2\le 2g(t)\le \lambda_{\max}(P)|\xlin|^2$, it follows that
\begin{equation*}
|\xlin(t)|\le \kappa e^{-\alpha t}|\xlin(0)|,
\end{equation*}
where
\begin{equation}
\label{eq:alpha_def}
\kappa =
\sqrt{\frac{\lambda_{\max}(P)}{\lambda_{\min}(P)}}
\exp\!\left(\frac{\mu}{\lambda_{\max}(P)}\right),
\alpha=
\frac{\mu}{T\lambda_{\max}(P)}.    
\end{equation}
Since \(P\succ0\) and \(\mu>0\), this choice satisfies \(\kappa\ge 1\) and \(\alpha\ge 0\).
Hence, the origin is exponentially stable \cite{ khalil2002nonlinear,mct}, which completes the proof.  
\end{proof}
We next extend Corollary~\ref{cor:exp_stable} to account for bounded input perturbations, which arise in practice due to estimation errors and implementation uncertainties.
\begin{corollary}
[Robustness to Estimation Errors]
\label{cor:robustness}
Suppose the conditions of Corollary~\ref{cor:exp_stable} hold. 
Let \(\widehat{x\dot{x}}\) and \(\widehat{\dot{x}}\) denote the event-based estimates of \(x\dot{x}\) and \(\dot{x}\), respectively, used by the controller \eqref{eq:control_law}, satisfying
\[
|\widehat{x\dot{x}} - x\dot{x}| \le \varepsilon_q,
\qquad
|\widehat{\dot{x}} - \dot{x}| \le \varepsilon_\ell,
\]
as guaranteed by Theorem~\ref{thm:net_count_bound}.
Suppose further that, in a neighborhood of the target orbit where the local
linearization is valid,
\[
|x(t)\dot{x}(t)| \le h_{\max},
\qquad
|\dot{x}(t)| \le v_{\max}.
\]
Then the tracking error satisfies
\begin{equation*}
    \limsup_{t\to\infty}\|\tilde{\xi}(t)\|
    \leq
    \frac{\kappa K}{\alpha}
    \left(
    h_{\mathrm{max}}\varepsilon_\ell
    +
    v_{\mathrm{max}}\varepsilon_q
    +
    \varepsilon_q\varepsilon_\ell
    \right),
\label{eq:robust_bound}
\end{equation*}
where \(\kappa\ge 1\) and \(\alpha>0\) are defined in \eqref{eq:alpha_def}. Consequently, the closed-loop trajectory converges to a neighborhood of the target periodic orbit \(\xi^\star(t)\), with radius proportional to the event-based estimation errors.
\end{corollary}

\begin{proof}
Let
\(
\widehat{x\dot{x}}=x\dot{x}+e_q(t), \:
\widehat{\dot{x}}=\dot{x}+e_\ell(t),
\)
where \(e_q(t)\) and \(e_\ell(t)\) denote the estimation errors associated with
the quadratic and linear event kernels, respectively.
By Theorem~\ref{thm:net_count_bound},
\[
|e_q(t)|\le \varepsilon_q,\quad
|e_\ell(t)|\le \varepsilon_\ell .
\]

Since the event-based implementation of the controller \eqref{eq:control_law} uses estimates of
\(x\dot{x}\) and \(\dot{x}\) rather than their exact values,
the nonlinear feedback
term \(x\dot{x}^2\) is implemented as
\[
(\widehat{x\dot{x}})\;(\widehat{\dot{x}})
=
(x\dot{x}+e_q)(\dot{x}+e_\ell).
\]
Therefore, the implemented input can be written as
$u(t) = u_{\mathrm{nom}}(t) + u_d(t)$ where $u_{\mathrm{nom}}(t)$ is the ideal control law~(8) and
\begin{equation*}
    u_d(t) = -\frac{K}{p_2}
    \bigl[
    (x\dot{x})\,e_\ell + \dot{x}\,e_q + e_q e_\ell
    \bigr].
\end{equation*}
Using
\[
|x(t)\dot{x}(t)|\le h_{\max},
|\dot{x}(t)|\le v_{\max},
|e_q(t)|\le \varepsilon_q,
|e_\ell(t)|\le \varepsilon_\ell,
\]
we obtain
\[
\|u_d\|_\infty
\le
\frac{K}{p_2}
\left(
h_{\max}\varepsilon_\ell
+
v_{\max}\varepsilon_q
+
\varepsilon_q\varepsilon_\ell
\right).
\]
Theorem~\ref{thm:net_count_bound} provides explicit bounds on
\(\varepsilon_q\) and \(\varepsilon_\ell\) in terms of the event accumulation
window, robot velocity, acceleration, and spatial profile parameters. The
constants \(h_{\max}\) and \(v_{\max}\) bound the corresponding kinematic
quantities in the local neighborhood of the target orbit. 
For trajectories evolving within a local neighborhood of the target orbit,
with
\[
|\tilde x(t)|\le \Delta_x,
\qquad
|\dot{\tilde{x}}(t)|\le \Delta_{\dot{x}},
\]
the bounds may be chosen as
\[
v_{\max}=a\omega+\Delta_{\dot{x}},
\]
and
\[
h_{\max}
=
\frac{a^2\omega}{2}
+
a\Delta_{\dot{x}}
+
a\omega\Delta_x
+
\Delta_x\Delta_{\dot{x}}.
\]
Substituting \(u(t)=u_{\mathrm{nom}}(t)+u_d(t)\) into \eqref{eq:qcar_longitudinal_dynamics} and linearizing about the target orbit as in \eqref{eq:LTV} gives
\[
\dot{\xlin}
=
A(t)\xlin
+
\begin{bmatrix}
0\\
p_2u_d(t)
\end{bmatrix},
\]
where \(A(t)\) is defined in \eqref{eq:A_def}. Let \(\Phi(t,\tau)\) denote the state-transition matrix of the nominal system \(\dot{\xlin}=A(t)\xlin\). By Corollary~\ref{cor:exp_stable}, there exist constants \(\kappa\ge 1\) and \(\alpha>0\) such that
\[
\|\Phi(t,\tau)\|
\le
\kappa e^{-\alpha(t-\tau)},
\qquad t\ge \tau\ge 0.
\]
By the variation-of-constants formula,
\[
\xlin(t)
=
\Phi(t,0)\xlin(0)
+
\int_0^t
\Phi(t,\tau)
\begin{bmatrix}
0\\
p_2u_d(\tau)
\end{bmatrix}
d\tau .
\]
Taking norms gives
\[
\begin{aligned}
\|\xlin(t)\|
&\le
\kappa e^{-\alpha t}\|\xlin(0)\|
+
\kappa p_2
\int_0^t
e^{-\alpha(t-\tau)}
|u_d(\tau)|\,d\tau \\
&\le
\kappa e^{-\alpha t}\|\xlin(0)\|
+
\kappa p_2
\|u_d\|_{\infty,[0,t]}
\int_0^t e^{-\alpha(t-\tau)}\,d\tau .
\end{aligned}
\]
Since
\[
\int_0^t e^{-\alpha(t-\tau)}\,d\tau
=
\frac{1-e^{-\alpha t}}{\alpha}
\le
\frac{1}{\alpha},
\]
we obtain
\[
\|\xlin(t)\|
\le
\kappa e^{-\alpha t}\|\xlin(0)\|
+
\frac{\kappa p_2}{\alpha}
\|u_d\|_{\infty,[0,t]}.
\]
Taking the limit superior as \(t\to\infty\) yields
\[
\limsup_{t\to\infty}\|\xlin(t)\|
\le
\frac{\kappa p_2}{\alpha}\|u_d\|_\infty.
\]
Substituting the bound on \(\|u_d\|_\infty\) gives
\[
\limsup_{t\to\infty}\|\xlin(t)\|
\le
\frac{\kappa K}{\alpha}
\left(
h_{\max}\varepsilon_\ell
+
v_{\max}\varepsilon_q
+
\varepsilon_q\varepsilon_\ell
\right).
\]
Thus, the linearized tracking-error dynamics are input-to-state stable with
respect to bounded event-based estimation errors \cite{mct}. Consequently, the closed-loop trajectory is practically orbitally stable, converging to a neighborhood of
the target periodic orbit whose radius scales with the estimation-error bounds.
\end{proof}

\section{Experiments and Simulations}
This section details the implementation and experimental validation of the proposed EBVS framework.

\subsection{Motion Dynamics Parameters}
The lumped parameters $p_1$, $p_2$, and $p_3$ in \eqref{eq:qcar_longitudinal_dynamics} were identified using the data-driven method from \cite{caponio2024modeling}. Due to a slight drivetrain asymmetry, the parameters differ between forward $(p_1,p_2,p_3) = (2.530,\,33.977,\,1.349)$ and backward $(2.954,\,37.497,\,-1.510)$ motion.

The proposed active sensing controller smoothly switches between directional parameters using the transition function $\tanh(\beta \dot{x})$ with \revision{with $\beta = 1000$}.

\subsection{Net Event Count Estimator Results}
To validate the theoretical estimators derived in Section \ref{sec:event_camera_model_and_event_processing}, a dual-pattern stimulus was presented to the event camera, as illustrated in Figure~\ref{fig:pattern_and_event_camera_snapshot}. The stimulus comprises a quadratic intensity profile in the upper region and a linear profile in the lower region of the display. By defining two spatially distinct kernels, $\mathcal{K}_1$ and $\mathcal{K}_2$, corresponding to these regions, the net event count within $\mathcal{K}_1$ provides an estimate of the position-velocity product, $x(t)\dot{x}(t)$, while the net count within $\mathcal{K}_2$ independently estimates the instantaneous velocity, $\dot{x}(t)$.


\begin{figure}[t]
    \centering
    \includegraphics[width=0.98\linewidth]{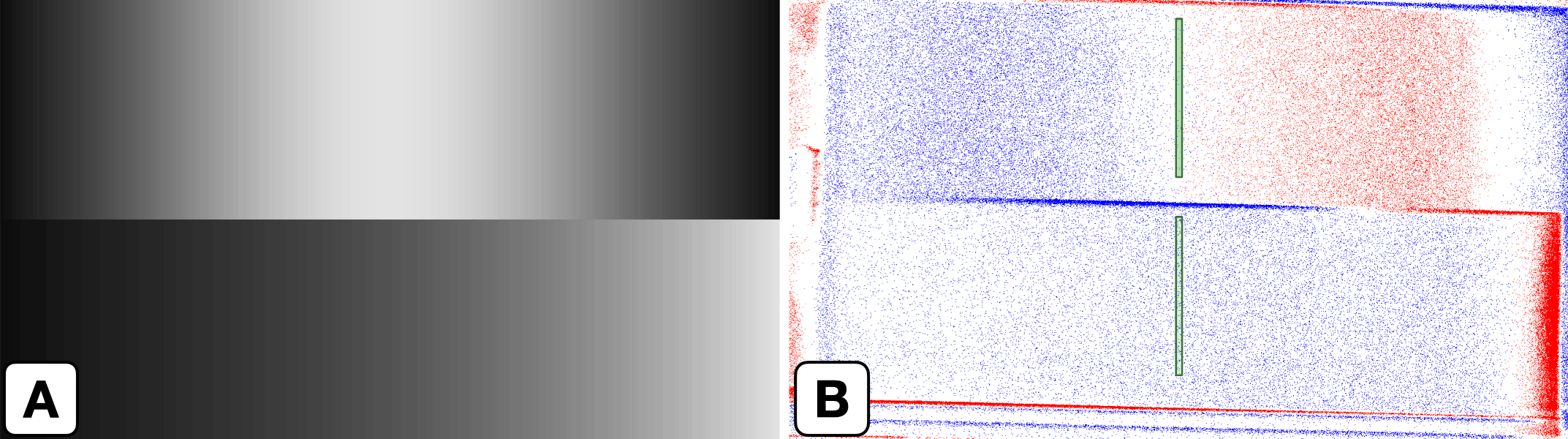}
    \caption{\revision{\textbf{(A)} Dual-pattern display with quadratic (top) and linear (bottom) intensity profiles. \textbf{(B)} Corresponding accumulated event stream, where red and blue dots indicate positive and negative polarity events, respectively and white background shows idle pixels (no event). The robot's motion to the left relative to the pattern generates the observed events. Green rectangles denote the kernels $\mathcal{K}_1$ and $\mathcal{K}_2$. Frame reconstruction is not required in our method; this figure is provided solely for illustration.}}
    \label{fig:pattern_and_event_camera_snapshot}
\end{figure}


\revision{To determine the lumped constants for the estimators $M_{\mathrm{Quadratic}}$ and $M_{\mathrm{Linear}}$ defined in \eqref{eq:M_Quadratic} and \eqref{eq:M_linear}, where the lumped constants correspond to $\alpha_q:=C_q \Delta t$ and $\alpha_l:=C_l \Delta t$, respectively, we employed a data-driven approach. Specifically, we model the net event counts as $N_{\mathrm{net}}^{(q)}(t) \approx \alpha_q\, x(t)\dot{x}(t)$ and $N_{\mathrm{net}}^{(l)}(t) \approx \alpha_l\, \dot{x}(t)$, and identify the constants by minimizing the mean squared error over synchronized time-series data. A median filter is applied to suppress spurious spikes, and the resulting one-parameter least-squares problems are solved via numerical optimization.} Figure~\ref{fig:net_event_count_estimator_results} illustrates the resulting estimates alongside their ground-truth values and the theoretical error bounds in \eqref{eq:net_event_estimator}. In this experiment, the robot executed a stochastic oscillatory motion to evaluate estimator performance across a broad operational envelope. During each forward and backward phase, the maximum velocity and target position were randomized. This approach ensures that the net event count estimators are validated against a diverse range of kinematic profiles.

\begin{figure}[t]
    \centering
    \includegraphics[width=0.98\linewidth]{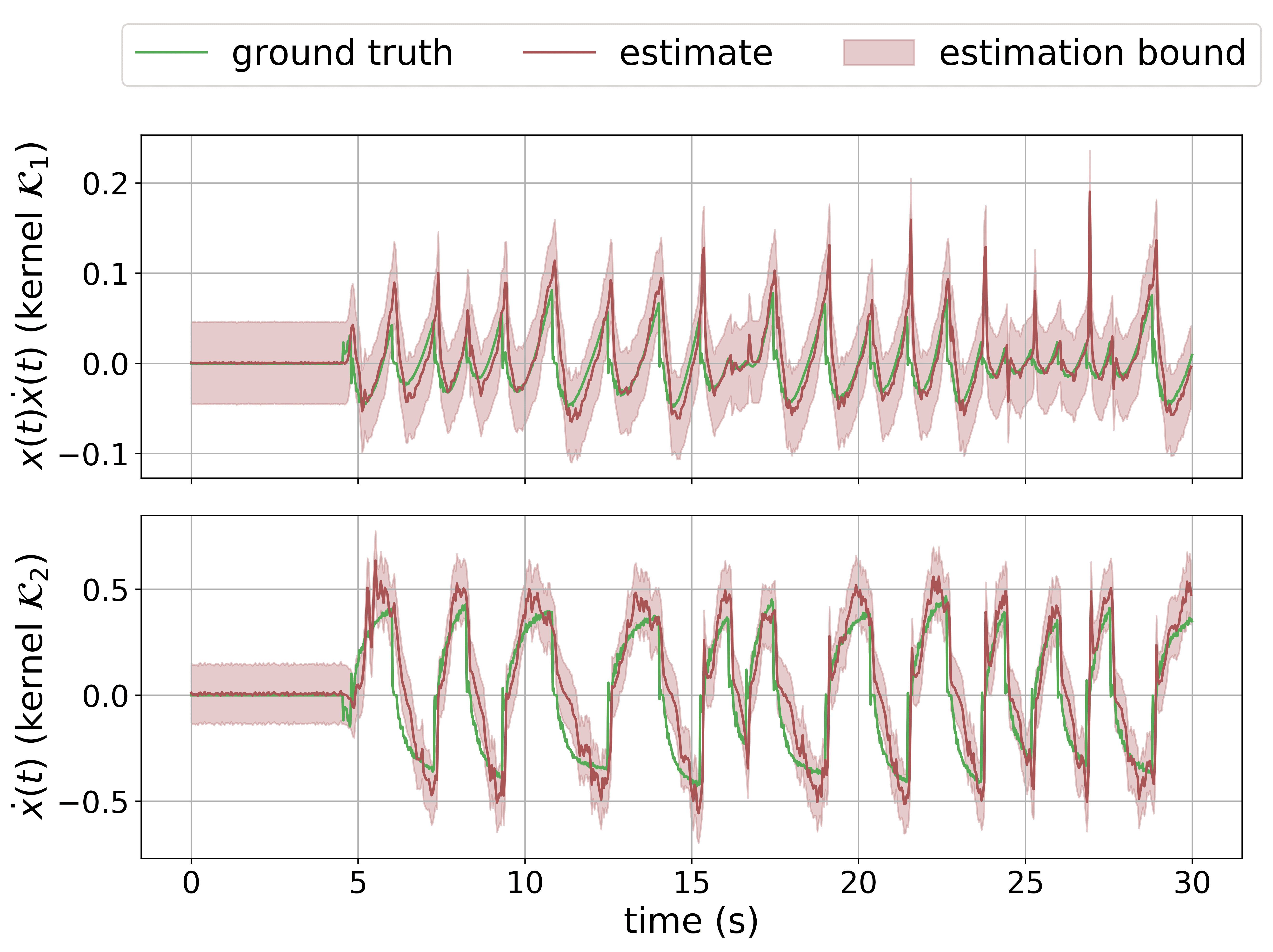}
    \caption{Event-based state estimation vs.\ ground truth. Top: $x(t)\dot{x}(t)$ from $\mathcal{K}_1$. Bottom: $\dot{x}(t)$ from $\mathcal{K}_2$. Shaded regions indicate theoretical bounds. Ground truth for the robot's position is provided by an array of six overhead motion capture cameras, while ground truth velocity is measured directly via the robot's onboard wheel encoders.}
    \label{fig:net_event_count_estimator_results}
\end{figure}

To validate the robustness of the proposed net event count estimator across various operational conditions, we evaluated its performance while varying the spread of the quadratic profile ($\sigma$), the slope coefficient of the linear profile ($k$), and the maximum allowed robot velocity $v_{\max}$,  which also affects the maximum acceleration $a_{\max}$ through the actuator limits. The results of these scenarios are summarized in Table \ref{tab:open_loop_experiments}. Estimation accuracy is quantified using the length-normalized $L_2$ norm of the discrete error vectors. Specifically, let $e \in \mathbb{R}^{N_s}$ represent the sample-wise difference between the estimated and ground-truth signals over $N_s$ total samples. The average estimation errors for the quadratic and linear profiles, denoted as $\|e_{\mathrm{q}}\|$ and $\|e_{\mathrm{l}}\|$ respectively, are formally calculated as $\frac{1}{N_s}\|e\|_2$. 
In these trials, the maximum robot velocity $v_{\max}$ was systematically varied because it can be directly regulated via the underlying motor control commands. As demonstrated in the table, the magnitude of the estimation errors $\|e_{\mathrm{q}}\|$ and $\|e_{\mathrm{l}}\|$ exhibits a strong, direct dependence on $v_{\max}$ and $a_{\max}$. Conversely, the errors remain largely invariant to changes in the displayed pattern parameters, specifically the spread $\sigma$ and the linear slope $k$. This empirical observation validates our proposed theoretical framework, confirming that the estimator's accuracy is fundamentally bounded by the temporal kinematics of the sensor rather than the spatial characteristics of the target.

\begin{table}[t]
    \centering
        \scriptsize
     \caption{Estimation Errors by Target Profile and Velocity}
    \label{tab:open_loop_experiments}
    \begin{tabular}{|c|c|c|c|c|c|c|}
        \hline
        \multirow{2}{*}{\textbf{Trial}} & \multicolumn{4}{c|}{\textbf{Parameters}} & \multicolumn{2}{c|}{\textbf{Results}} \\ \cline{2-7} 
         & $\sigma$ & $k (\times 10^{-3})$ & $v_{\max}$ & $a_{\max}$ ($\approx$) & $\|e_{\mathrm{q}}\|$ & $\|e_{\mathrm{l}}\|$ \\ \hline
        1 & $330$ & $1.299 $ & $0.45$ & $2$ & $0.000807$ & $0.004303$ \\ \hline
        2 & $330$ & $1.299 $ & $0.65$ & $3$ & $0.001573$ & $0.006937$ \\ \hline
        3 & $330$ & $2.205 $ & $0.45$ & $2$ & $0.000649$ & $0.003578$ \\ \hline
        4 & $330$ & $2.205 $ & $0.65$ & $3$ & $0.001657$ & $0.007203$ \\ \hline
        5 & $430$ & $1.299 $ & $0.45$ & $2$ & $0.000994$ & $0.004303$ \\ \hline
        6 & $430$ & $1.299 $ & $0.65$ & $3$ & $0.001491$ & $0.006356$ \\ \hline
        7 & $430$ & $2.205 $ & $0.45$ & $2$ & $0.000985$ & $0.004740$ \\ \hline
        8 & $430$ & $2.205 $ & $0.65$ & $3$ & $0.001532$ & $0.007376$ \\ \hline
    \end{tabular}
\end{table}

\begin{figure*}[t]
    \centering
    \includegraphics[width=0.98\linewidth]{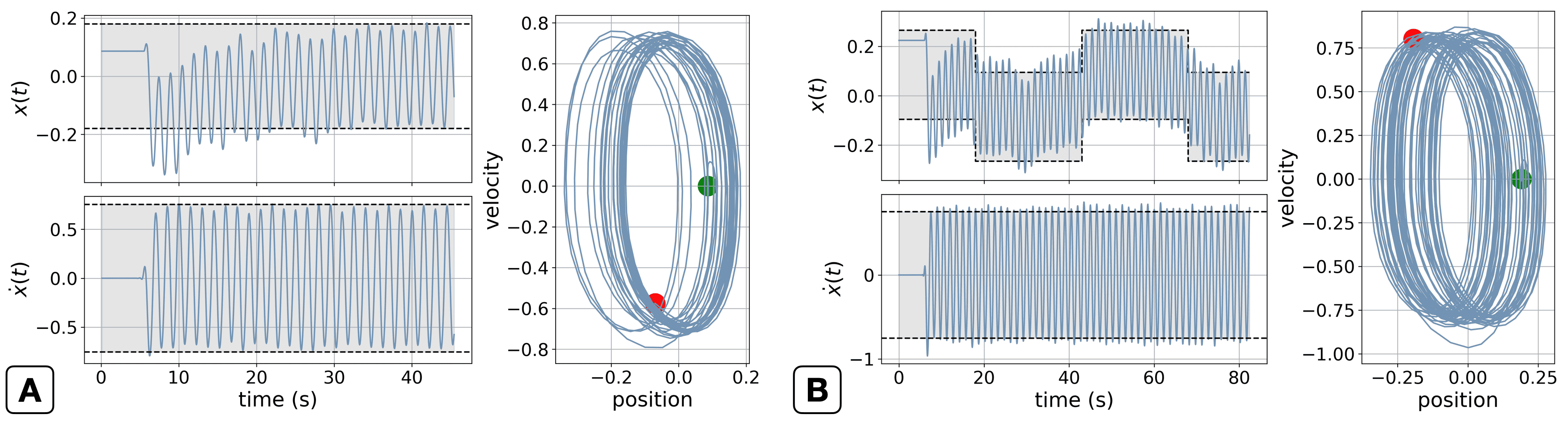}
    \caption{Closed-loop EBVS response for \textbf{(A)} fixed and \textbf{(B)} time-varying stabilization points. In both cases, in the left panel, the top plot shows position and the bottom plot shows velocity, with the desired oscillation radius $a$ indicated by dashed black lines. The right panels present the phase portraits, where green and red markers denote the start and end of the trajectory, respectively.}
    \label{fig:closed_loop_plots_together}
\end{figure*}

\subsection{EBVS with Active Sensing}
Using the dual-displayed pattern, we estimate $x(t)\dot{x}(t)$ and $\dot{x}(t)$ directly from the net event counts to generate feedback for the control law \eqref{eq:control_law}. Figure~\ref{fig:closed_loop_plots_together}(A) shows the closed-loop response, where the robot stabilizes around the center of the quadratic profile (i.e., $x=0$). For this trial, the reference amplitude and oscillation frequency of the periodic motion are set to $a=0.18$ and $\omega = 2\pi/1.5$, respectively. Also, the output gain is set to $K=1.5$, yielding $\delta=0.0486$.

Figure~\ref{fig:closed_loop_plots_together}(B) illustrates a scenario in which the center of the quadratic profile is shifted over time, thereby changing the stabilization point (i.e., $x=0$). Specifically, the target is switched between offsets of $-0.085$~m and $0.085$~m relative to the monitor center. For clarity in reporting, the ground-truth position is expressed in a fixed global frame defined at the center of the monitor, as measured by the motion capture system, even though the controller origin varies over time. For this trial, the same parameter values are used.

\revision{
\begin{remark}[Worst case robustness bounds]
Using the experimental parameters \(a=0.18\), \(\omega=2\pi/1.5\), \(K=1.5\), and the identified forward-motion parameter \(p_1=2.530\), we obtain \(\delta=Ka^2=0.0486\). Choosing the Lyapunov parameter as the critical value \(\eta=\eta^\dagger\) from Corollary~\ref{cor:exp_stable} to guarantee positivity of \(Q(t)\) yields \(\eta^\dagger = 0.401224\), \(\mu = 0.104936\), \(\alpha = 0.116941\), and \(\kappa = 2.638806\). 

Using the experimentally measured bounds \(h_{\max} = 0.0813\), \(v_{\max} = 0.4567\), \(\varepsilon_q = 0.0454\), and \(\varepsilon_\ell = 0.14\), Corollary~\ref{cor:exp_stable} gives
\[
\limsup_{t\to\infty}\|\xlin(t)\|
\le 1.3029.
\]
This estimate guarantees practical orbital stability of the closed-loop system under the experimentally observed event-estimation errors. The bound is highly conservative, since it is derived using worst-case sup-norm estimates and Lyapunov inequalities, but it demonstrates that bounded event-estimation errors induce bounded tracking errors without destroying stability of the target periodic orbit.
\end{remark}
}

\section{Conclusion and Future Work}
\vspace{-0.2em}
This paper presented a purely event-driven, \revision{1D proof-of-concept}, bio-inspired visual servoing framework. We showed that specific spatial profiles isolate distinct \revision{combinations of kinematic states}, enabling direct nonlinear state feedback and maintaining observability at equilibrium via a limit-cycle active sensing controller. The current approach relies on a specially designed visual scene and simple rectangular kernels; future work will relax this assumption using richer kernel geometries as in prior work for traditional cameras~\cite{kallem2007kernel}, extend the framework to multi-dimensional servoing, and explore robustness against dynamic targets in unstructured environments.

\revision{Another future direction is extending the controller to an adaptive formulation to handle model uncertainties, similar to approaches in IBVS~\cite{xing2024field}.}

\section{Acknowledgment}
This work was supported by the Office of Naval Research, United States, under grant N00014-21-1-2431 to NJC and MS, the National Institutes of Health, United States, under grant R01NS147767 to NJC, and the U.S. National Science Foundation under grant award 2208182 to MS.

\IEEEtriggeratref{26}
\bibliographystyle{IEEEtran}
\begin{spacing}{1}
\bibliography{refs.bib}  
\end{spacing}

\end{document}